\documentclass[letterpaper, 10 pt, conference]{ieeeconf}  

\IEEEoverridecommandlockouts                              

\usepackage{graphics} 
\usepackage{epsfig} 
\usepackage{mathptmx} 
\usepackage{times} 
\usepackage{amsmath} 
\usepackage{amssymb}  
\usepackage{censor}

\usepackage[ruled,vlined,linesnumbered]{algorithm2e}

\usepackage[caption=false,font=footnotesize]{subfig}

\usepackage{multirow}

\newtheorem{definition}{Definition}

\newtheorem{assumption}{Assumption}
\newtheorem{problem}{Problem}
\newtheorem{lemma}{Lemma}
\newtheorem{theorem}{Theorem}

\DeclareMathOperator*{\argmin}{arg\,min}

\usepackage{booktabs}
\DeclareMathAlphabet{\mathcal}{OMS}{cmsy}{m}{n}

\title{\LARGE \bf
Correct-by-Construction Behavior Tree Synthesis from Signal Temporal Logic Specifications with Application to Robotic Missions
}

\author{Jiaheng Dong, Jingyi Huang and Liang Han$^*$
\thanks{This work was supported by the National Natural Science Foundation of
China under Grant 62373024, and the Fundamental Research Funds for the
Central Universities. Jiaheng Dong, Jingyi Huang and Liang Han are with the Sino-French Engineer School,Beihang University, Beijing 100191, China. (Corresponding Author: Liang Han.){\tt\small E-mail: $\{$Dongjiaheng24, liang$\_$han$\}$@buaa.edu.cn}}
}

\begin{document}

\maketitle
\thispagestyle{empty}
\pagestyle{empty}

\begin{abstract}

Behavior Trees (BTs) are widely adopted for complex task execution in robotics, providing modular, reactive control but lacking formal guarantees. However, existing correct-by-construction synthesis from Linear Temporal Logic (LTL) cannot express quantitative timing constraints. This letter synthesizes correct-by-construction BTs from Signal Temporal Logic (STL) specifications. The workspace is modeled as a timed transition system and abstracted into a zone graph, and an augmented state space tracking both logical progress and timing constraints is introduced. A hierarchical fixed-point algorithm computes winning sets for an STL fragment encompassing safety, reachability, response, recurrence, and persistence, yielding BT subtrees with a runtime constraint function. Correctness guarantees are proven and complexity bounds are derived. Simulations demonstrate specification satisfaction with strictly positive robustness, and a physical quadrotor experiment with six STL specifications validates practical deployability.


\end{abstract}

\section{INTRODUCTION}

Behavior Trees (BTs) are increasingly recognized as a powerful architecture for modular and reactive robot control, and have been deployed across diverse domains including autonomous vehicles, industrial manufacturing, and service robotics \cite{iovino2022survey, dominguez2022stack}. Their hierarchical structure supports task decomposition, preemption, and human readability, making them preferable to finite state machines (FSMs) in many scenarios \cite{iovino2025comparison, biggar2021expressiveness}. However, most BT designs are still manually crafted by human experts, and their correctness with respect to formal specifications remains unverified until deployment.

To address the need for formal guarantees, automated verification \cite{biggar2020framework} and synthesis of correct-by-construction BTs have gained significant attention. Notably, the synthesis of BTs from Linear Temporal Logic (LTL) specifications has been studied in \cite{colledanchise2017synthesis}, where LTL is used to automatically generate BTs that satisfy given agent specifications over finite transition systems. This approach guarantees task completion for specifications expressible in an LTL fragment. However, LTL is inherently qualitative and cannot express quantitative timing constraints. In practice, robotic missions often impose temporal deadlines on tasks like charging, alarm response, and persistent monitoring, which fundamentally fall outside the expressiveness of standard LTL.


Signal Temporal Logic (STL) extends LTL with quantitative timing constraints via real-valued intervals \cite{maler2004monitoring}, supported by established monitoring \cite{donze2010robust} and robustness \cite{deshmukh2017robust} frameworks. However, synthesizing reactive controllers from STL remains an open challenge. Model Predictive Control (MPC) encodes STL as mixed-integer programs \cite{sadraddini2015robust}, yet its receding horizon limits scalability to persistent tasks. Control Barrier Function (CBF) methods enforce STL through continuous-time invariance \cite{lindemann2018control, nawaz2024reactive}, but lack discrete task-level modularity. Reinforcement learning with STL reward shaping \cite{wang2025multi} produces uninterpretable black-box policies without formal guarantees. Automata-based reactive synthesis provides correct-by-construction guarantees via two-player temporal logic games \cite{liu2013synthesis}, yet yields a monolithic finite automaton that entangles all specifications into a single opaque state machine, precluding per-task interpretability and incremental modification. None of these approaches simultaneously offers formal guarantees, modular multi-specification decomposition, and reactive long-horizon execution. BTs naturally fill this gap, as their hierarchical tick-based composition mirrors the conjunctive structure of multi-specification tasks, producing per-specification subtrees that are independently traceable, incrementally extensible, and directly deployable in established robot BT frameworks~\cite{colledanchise2021implementation}.

\begin{figure}[t]
    \centering
    \includegraphics[width=\linewidth]{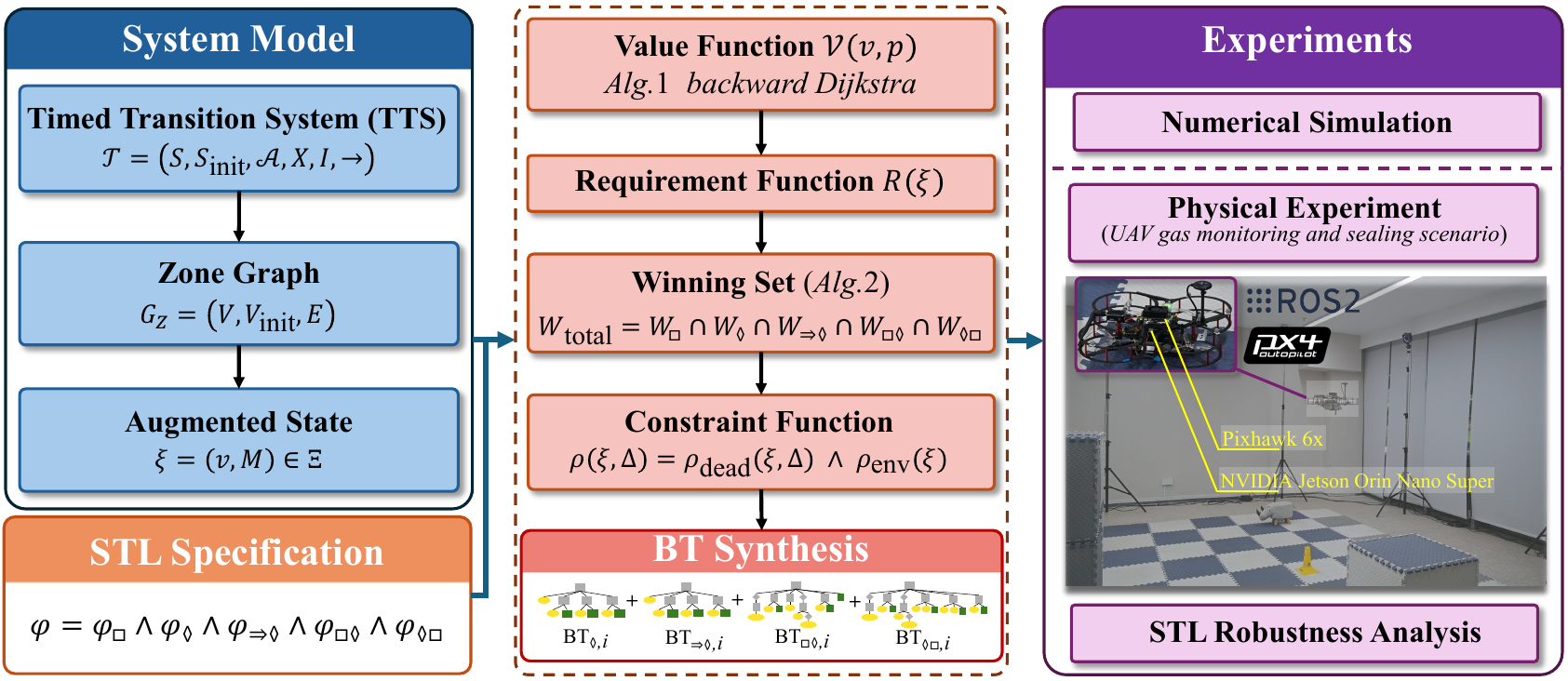}
    \caption{Overview of the proposed synthesis framework, from system modeling and STL specification through offline BT synthesis to online deployment validated in simulation and physical experiments.}
    \label{fig:workflow}
\end{figure}

This letter introduces a formal framework for synthesizing correct-by-construction BTs from STL specifications over timed transition systems, as illustrated in Fig.~\ref{fig:workflow}. The workspace is modeled as a timed transition system, abstracted into a finite zone graph with an augmented state space. A value function estimates the time-to-reach each target proposition, guiding a requirement function that evaluates per-specification feasibility. A hierarchical fixed-point algorithm computes winning sets for an STL fragment covering safety, reachability, response, recurrence, and persistence, and a constraint function confines the synthesized BT, composed of per-specification subtrees, to this winning set, ensuring formal satisfaction under adversarial conditions. The main contributions of this letter are as follows:

\begin{itemize}
    \item A synthesis framework is proposed that maps STL specifications to Behavior Trees, enabling robots to satisfy specifications with quantitative timing constraints.
    \item Theoretical analysis is provided, including correct-by-construction guarantees for the synthesized Behavior Trees and computational complexity bounds.
    \item The proposed framework is validated through numerical simulations and a physical quadrotor flight experiment, demonstrating its feasibility, robustness under environmental disturbances, and practical deployability.
\end{itemize}

\section{PRELIMINARIES}

\subsection{Behavior Trees}

\begin{definition}[Behavior Tree]\label{def:bt}
A \emph{Behavior Tree} (BT) is a tuple $\mathcal{B} = (\mathcal{N}, r, \mathrm{type}, \mathrm{child})$, where $\mathcal{N}$ is the set of nodes, $r \in \mathcal{N}$ is the root, $\mathrm{type}\colon \mathcal{N} \to \{\mathrm{Seq}, \mathrm{Fal}, \mathrm{Dec}, \mathrm{Act}, \mathrm{Cond}\}$ assigns a node type, and $\mathrm{child}\colon \mathcal{N} \to \mathcal{N}^*$ maps each internal node to an ordered list of children. At every \emph{tick}, the root initiates a depth-first traversal in which each node returns a status in $\{S, F, R\}$ according to its type-specific semantics~\cite{colledanchise2018behavior}.
\end{definition}

This letter uses the following node types. The \emph{sequence} node ($\rightarrow$) ticks children left to right, returning S only if all return S, and F or R upon the first such child. The \emph{fallback} node ($?$) returns F only if all return F, and S or R upon the first such child. The \emph{decorator} node (diamond) modifies its child's status: $\mathrm{S{\to}F}$ converts S to F, $\mathrm{S{\to}R}$ converts S to R, and $\mathrm{Inv}$ swaps S and F. The leaf \emph{action} node (rectangle) executes robot behaviors, returning R until completion (S) or failure (F). The leaf \emph{condition} node (ellipse) instantaneously checks propositions, returning S or F.

\subsection{Signal Temporal Logic}

Signal Temporal Logic (STL) extends linear temporal logic operators with real-valued time intervals. Given a set of atomic propositions $\mathit{AP}$, the syntax for an STL formula $\varphi$ is given by the following grammar:
\begin{equation}\label{eq:stl_grammar}
  \varphi \triangleq  p \mid \lnot\varphi \mid \varphi_1 \land \varphi_2 \mid \varphi_1 \mathcal{U}_{[a,b]} \varphi_2,
\end{equation}
where $\varphi_1$ and $\varphi_2$ are STL formulas, $p \in \mathit{AP}$, and $\mathcal{U}_{[a,b]}$ denotes the bounded \emph{until} operator over $[a,b] \subset \mathbb{R}_{+}$. Derived operators include bounded \emph{eventually} ($\Diamond_{[a,b]}\,\varphi = \mathrm{true}\;\mathcal{U}_{[a,b]}\,\varphi$), bounded \emph{always} ($\Box_{[a,b]}\,\varphi = \lnot\Diamond_{[a,b]}\,\lnot\varphi$), disjunction ($\lor$), and implication ($\Rightarrow$); unbounded counterparts use $[0, \infty)$.


\begin{definition}[STL Semantics]\label{def:stl_sem}
Given a signal $\sigma\colon \mathbb{R}_+ \to 2^{\mathit{AP}}$, the satisfaction relation $(\sigma, t) \models \varphi$ is:
\begin{align}
(\sigma, t) \models p \;&\Leftrightarrow\; p \in \sigma(t), \nonumber\\
(\sigma, t) \models \lnot\varphi \;&\Leftrightarrow\; (\sigma,t) \not\models \varphi, \nonumber\\
(\sigma, t) \models \varphi_1 {\land} \varphi_2 \;&\Leftrightarrow\; (\sigma,t) \models \varphi_1 \;\text{and}\; (\sigma,t) \models \varphi_2, \nonumber\\
(\sigma, t) \models \varphi_1 \mathcal{U}_{[a,b]} \varphi_2 \;&\Leftrightarrow\; \exists\, t' {\in} [t{+}a,\, t{+}b]:\; (\sigma, t') \models \varphi_2 \nonumber\\[-2pt]
&\quad\;\; \land\; \forall\, t'' {\in} [t,\, t'):\; (\sigma, t'') \models \varphi_1, \nonumber
\end{align}
where $\varphi, \varphi_1, \varphi_2$ are STL formulas. A signal satisfies $\varphi$, written $\sigma \models \varphi$, iff $(\sigma, 0) \models \varphi$.
\end{definition}

\section{PROBLEM FORMULATION}

This section formalizes the system as a timed transition system (TTS) with a zone graph abstraction, then defines the five-type STL specification and the augmented state, and finally states the synthesis problem.

\subsection{System Model}
The robot and its environment are modeled as a \emph{timed transition system} (TTS) $\mathcal{T} = (S, S_{\mathrm{init}}, \mathcal{A}, X, I, {\to})$, where $S$ is a finite set of locations, $S_{\mathrm{init}} \subseteq S$ is the set of initial locations, $\mathcal{A}$ is a finite set of actions, $X$ is a finite set of clocks, $I\colon S \to \Phi(X)$ assigns invariant constraints, and ${\to}$ defines the discrete transitions $s \xrightarrow{g,a,r} s'$ with guard $g$ and clock resets $r \subseteq X$. To make the infinite continuous state space tractable, the timed transition system $\mathcal{T}$ is abstracted into a finite \emph{zone graph} $G_Z = (V, V_{\mathrm{init}}, E)$, where each node $v=(s, D) \in V$ pairs a discrete location $s$ with a clock zone $D$ represented by a Difference Bound Matrix (DBM). The edges $E \subseteq V \times \mathcal{A} \times V$ encode the successor zones generated by alternating \emph{time transitions} and \emph{discrete transitions}~\cite{kwiatkowska2011prism}. During time transitions, all clocks advance uniformly while satisfying the invariant, whereas during discrete transitions, guards are verified and specified clocks are reset.

\subsection{STL Specification}
The agent specification $\varphi$ is defined as a conjunction of five STL formula types:
\begin{equation}\label{eq:spec_conj}
    \varphi = \varphi_{\Box} \land \varphi_{\Diamond} \land \varphi_{\Rightarrow\Diamond} \land \varphi_{\Box\Diamond} \land \varphi_{\Diamond\Box},
\end{equation}
where each specification type is indexed by $\kappa \in \mathcal{K} \triangleq \{ \Box,\, \Diamond,\, {\Rightarrow}{\Diamond},\, \Box\Diamond,\, \Diamond\Box \}$ and comprises one or more STL formulas, $\varphi_{\kappa} = \bigwedge_i \varphi_{\kappa,i}$. To simplify notation, an index $n \in \{1,\ldots,5\}$ is assigned to each type $\kappa$ following the order in $\mathcal{K}$, where parameters $p_{\kappa,i}$ and $q_{\kappa,i}$ are respectively denoted as $p_{n,i}$ and $q_{n,i}$. Table~\ref{tab:stl_fragment} defines the grammar and semantics of each $\varphi_{\kappa,i}$, with $p_{\kappa,i}$ denoting a controllable atomic proposition and $q_{\kappa,i}$ an uncontrollable atomic proposition.


\begin{table}[h]
  \centering
  \caption{STL Specification Formula and Semantics}
  \label{tab:stl_fragment}
  \renewcommand{\arraystretch}{1.2}
  \resizebox{\columnwidth}{!}{%
  \begin{tabular}{@{}lll@{}}
    \toprule
    \textbf{Type} & \textbf{Formula} & \textbf{Semantics} \\
    \midrule
    Safety & $\varphi_{\Box,i} = \Box_{[l_{1,i},\, u_{1,i}]} \, p_{1,i}$ & $p_{1,i}$ holds throughout $[l_{1,i},\, u_{1,i}]$ \\
    Reachability & $\varphi_{\Diamond,i} = \Diamond_{[l_{2,i},\, u_{2,i}]}\, p_{2,i}$ & $p_{2,i}$ satisfied within $[l_{2,i},\, u_{2,i}]$ \\
    Response & $\varphi_{\Rightarrow\Diamond,i} = \Box\bigl(q_{3,i} \Rightarrow \Diamond_{[l_{3,i},\, u_{3,i}]}\, p_{3,i}\bigr)$ & If $q_{3,i}$, satisfy $p_{3,i}$ in $[l_{3,i},\, u_{3,i}]$ \\
    Recurrence & $\varphi_{\Box\Diamond,i} = \Box\bigl(q_{4,i} \Rightarrow \Box_{[0,\,\bar{u}_{4,i}]} \Diamond_{[0,\,\hat{u}_{4,i}]}\, p_{4,i}\bigr)$ & While $q_{4,i}$ for $\bar{u}_{4,i}$, visit $p_{4,i}$ every $\hat{u}_{4,i}$ \\
    Persistence & $\varphi_{\Diamond\Box,i} = \Box\bigl(q_{5,i} \Rightarrow \Diamond_{[0,\,\bar{u}_{5,i}]} \Box_{[0,\,\hat{u}_{5,i}]}\, p_{5,i}\bigr)$ & If $q_{5,i}$, maintain $p_{5,i}$ for $\hat{u}_{5,i}$ within $\bar{u}_{5,i}$ \\
    \bottomrule
  \end{tabular}
  }
\end{table}

\subsection{Augmented State}
Since the zone graph $G_Z$ encodes physical dynamics but lacks specification tracking, each node $v \in V$ is augmented with a memory state $M = (\mathbf{p}, \mathbf{b}, \mathbf{t})$. This yields an augmented state $\xi = (v, M) \in \Xi$, where $\Xi$ denotes the augmented state space. The memory state $M$ consists of components that capture logical completion, internal modes, and timing constraints: 1) Completion flags $\mathbf{p} = (p_{\mathrm{reach}},\, p_{\mathrm{resp}})$ track boolean completion status for reachability and response tasks; 2) Mode flags $\mathbf{b} = (b_{\mathrm{rec}}, b_{\mathrm{act}}, b_{\mathrm{hld}})$ encode internal modes (\text{IDLE}/\text{ACTIVE} for recurrence, and \text{IDLE}/\text{SEARCH}/\text{HOLD} for persistence); 3) Timers $\mathbf{t} = (x_{\mathrm{reach}},\, x_{\mathrm{resp}},\, z,\, y)$ record the elapsed time. To ensure a finite state space $\Xi$, every timer $x$ is saturated at its deadline $u$: $\text{sat}(x, u) = \min(x, u+1)$. Transitions in $\Xi$, denoted $\delta(\xi, a) = \xi', a\in\mathcal{A}$, jointly update the zone graph node $v$, advance saturated timers by the associated time transitions, and update the logic flags based on the atomic propositions satisfied in the new state.

\subsection{Problem Statement}
To synthesize correct-by-construction BTs against the agent specification $\varphi$, the following assumptions are made regarding the system and its environment.





\begin{assumption}[Non-Zeno]\label{assump:Non Zeno}
    Every location $s \in S$ admits an idle action that allows the system to remain at $s$ while time elapses. The system is non-Zeno: every execution produces only finitely many discrete transitions in bounded time.
\end{assumption}

\begin{assumption}[Initial Feasibility]\label{assump:Initial Consistency}
    The specification $\varphi$ is realizable from the initial augmented state $\xi_0$.
\end{assumption}

\begin{assumption}[Non-Blocking]\label{assump:No Deadlock}
    For every reachable augmented state $\xi \in \Xi$, there exists an action $a \in \mathcal{A}$ whose successor set is non-empty, i.e., $\delta(\xi, a) \neq \emptyset$.
\end{assumption}

\begin{assumption}[Bounded Adversarial Interference]\label{assump: Bounded Adversial}
    For any reachable $\xi \in \Xi$, the environment cannot indefinitely prevent the system from executing a discrete transition that changes location.
\end{assumption}






\begin{definition}[Induced Policy and Execution]\label{def:policy_trace}
Given a BT $\mathcal{B}$ over the augmented state space $\Xi$, the \emph{induced policy} $\pi_{\mathcal{B}}\colon \Xi \to \mathcal{A}$ maps each state $\xi$ to the action selected during the \emph{tick} of $\mathcal{B}$ in $\xi$. An \emph{execution} under $\pi_{\mathcal{B}}$ from $\xi_0$ is the sequence $(\xi_0, t_0),(\xi_1, t_1),\ldots$ with $t_0 {=} 0$, $t_{k+1} {\ge} t_k$, and $\xi_{k+1} \in \delta(\xi_k,\, \pi_{\mathcal{B}}(\xi_k))$. Each execution induces a signal $\hat{\sigma}\colon \mathbb{R}_+ \to 2^{\mathit{AP}}$ via the proposition labeling of zone graph locations. Denote $\pi \models \varphi$ iff $\hat{\sigma} \models \varphi$ for every execution from~$\xi_0$.
\end{definition}

\begin{problem}\label{prob:Problem_synthesis}
   Given a TTS $\mathcal{T}$ and its specification $\varphi$, under Assumptions~\ref{assump:Non Zeno}--\ref{assump: Bounded Adversial}, synthesize a Behavior Tree $\mathcal{B}$ whose induced policy $\pi_{\mathcal{B}}$ satisfies $\pi_{\mathcal{B}} \models \varphi$.
\end{problem}

The main challenge is the interplay between the large augmented state space $\Xi$ and the quantitative timing constraints embedded in $\varphi$: a correct policy must ensure both invariance and timely progress for each specification type under adversarial conditions.

\section{Proposed Method}
This section details the synthesis framework for Problem~\ref{prob:Problem_synthesis}. To establish correct-by-construction guarantees, the approach ensures safety via forward invariance and liveness via monotonic progress. This is organized in three steps: (1) formulating requirement functions, guided by a value function, to compute a hierarchical winning set; (2) applying a runtime constraint function to restrict transitions within this set; and (3) mapping the resulting policy into a behavior tree.

\subsection{Value Function}
The value function estimates the minimum time to reach a state satisfying a given proposition $p$ from any zone graph node, serving as a heuristic to guide action selection and verify timing feasibility.


\begin{definition}[Value Function]
    Given a zone graph $G_Z = (V, V_{\mathrm{init}}, E)$ and a proposition~$p$, the \emph{value function} $\mathcal{V}\colon V \times \mathit{AP} \to \mathbb{R}_{+} \cup \{\infty\}$ is defined as:
    \begin{equation}\label{eq:value_func}
        \mathcal{V}(v, p) = \begin{cases}
            0 & \text{if } s \models p\\[2pt]
            \displaystyle\min_{(v,a,v') \in E} \bigl\{c(v,v') + \mathcal{V}(v',p)\bigr\} & \mathrm{otherwise}
        \end{cases},
    \end{equation}
\end{definition}
where the cost $c(v,v') = w_e + \delta_e$ combines the action weight~$w_e$ with the minimum delay~$\delta_e$ within zone~$D$ needed to satisfy the guard of edge $e = (v,a,v')$. Since $\mathcal{V}$ satisfies the optimal substructure with non-negative costs, it can be efficiently computed for the entire state space via a backward Dijkstra, as shown in Algorithm~\ref{alg:value}. 


\begin{algorithm}[t]
\caption{\textsc{ComputeValueFunction}($G_Z$, $p$)}\label{alg:value}

\SetKwInOut{Input}{Input}
\SetKwInOut{Output}{Output}

\Input{Zone graph $G_Z = (V, V_{\mathrm{init}}, E)$, atomic proposition $p$}
\Output{$\mathcal{V}(v,p)$ for all $v \in V$}

$\mathcal{V}(v,p) \leftarrow \infty$ for all $v \in V$\;
$\mathcal{Q} \leftarrow \emptyset$ \tcp*{min-priority queue}

\ForEach{$v = (s, D) \in V$ with $s \models p$}{
    $\mathcal{V}(v,p) \leftarrow 0$;\quad $\mathcal{Q}.\mathrm{push}(v, 0)$\;
}

\While{$\mathcal{Q} \neq \emptyset$}{
    $(v', d) \leftarrow \mathcal{Q}.\mathrm{pop}()$\;
    \If{$d > \mathcal{V}(v',p)$}{
        \textbf{continue}\;
    }
    \ForEach{reverse edge $(v, a, v') \in E$}{
        \If{$d + c(v,v') < \mathcal{V}(v,p)$}{
            $\mathcal{V}(v,p) \leftarrow d + c(v,v')$;\quad $\mathcal{Q}.\mathrm{push}(v, \mathcal{V}(v,p))$\;
        }
    }
}
\Return $\mathcal{V}$\;
\end{algorithm}

\subsection{Requirement Function}

\begin{definition}[Requirement Function]
    The \emph{requirement f\-u\-n\-c\-t\-i\-o\-n} $R\colon \Xi \to \{\mathrm{true},\mathrm{false}\}$ determines whether the agent specification~$\varphi$ can still be satisfied from augmented state~$\xi$:
    \begin{equation}\label{eq:R_total}
        R(\xi) = R_{\Box}(\xi) \land R_{\Diamond}(\xi) \land R_{\Rightarrow\Diamond}(\xi) \land R_{\Box\Diamond}(\xi) \land R_{\Diamond\Box}(\xi).
    \end{equation}
\end{definition}

Each \emph{type-level} component is a conjunction of requirement functions over specifications of the same type:
\begin{equation}\label{eq:R_conj}
    R_{\kappa}(\xi) \triangleq \textstyle\bigwedge_i R'_{\kappa}(\xi,\, p_{\kappa,i}),
\end{equation}
where $\kappa \in \mathcal{K}$ and $R'_{\kappa}(\xi, p_{\kappa,i})$ evaluates a single specification $\varphi_{\kappa,i}$ with target proposition~$p_{\kappa,i}$. To systematically enforce $\varphi$, every requirement function $R'_{\kappa}$ evaluates the conjunction of three critical dimensions: 1) Logical Completion: whether the specification is already satisfied or currently inactive; 2) Timing Feasibility: whether reaching the target is physically possible within the remaining deadline, bounded by $\mathcal{V}(v,p)$; and 3) Recursive Invariance: whether the system can transition to a valid successor state.

For each $\varphi_{\kappa,i}$, $\kappa \in \mathcal{K}$, let the \emph{type-level winning set} $W_{\kappa,i} \subseteq \Xi$ denote the set of augmented states from which $\varphi_{\kappa,i}$ can be satisfied:
\begin{equation}\label{eq:W_type}
    W_{\kappa,i} = \bigl\{\xi \in \Xi \mid R'_{\kappa}(\xi,\, p_{\kappa,i}) = \mathrm{true} \bigr\}.
\end{equation}
Since each $R_{\kappa}(\xi)$ is the conjunction over specifications of the same type~\eqref{eq:R_conj}, the corresponding type-level winning set is $W_{\kappa} = \bigcap_i W_{\kappa,i}$. The total winning set aggregates all specifications and is given by:
\begin{equation}\label{eq:winning_total}
    W_{\mathrm{total}} = W_{\Box} \cap W_{\Diamond} \cap W_{\Rightarrow\Diamond} \cap W_{\Box\Diamond} \cap W_{\Diamond\Box}.
\end{equation}
By construction, $\xi \in W_{\mathrm{total}}$ if and only if $R(\xi) = \mathrm{true}$, indicating that all specifications can still be satisfied from~$\xi$.

\begin{definition}[Controllable Predecessor]
Given the augmented state space~$\Xi$, the \emph{controllable predecessor} operator $\operatorname{CPre} \colon 2^{\Xi} \to 2^{\Xi}$ maps a set $W \subseteq \Xi$ to the set of states from which the robot can guarantee remaining in~$W$ in one step:
\begin{equation}\label{eq:cpre-}
   \operatorname{CPre}(W) = \bigl\{\xi \in \Xi \;\big|\; \exists\, a \in \mathcal{A},\; \forall \xi' \in \delta(\xi,a), \xi' \in W\bigr\}. 
\end{equation}

\end{definition}

\textbf{Safety} $R_{\Box}$:
For $\varphi_{\Box,i} = \Box_{[l_{1,i},u_{1,i}]}\, p_{1,i}$ with elapsed time~$t$:
\begin{equation}\label{eq:req_safe}
    R'_{\Box}(\xi,\, p_{1,i}) = \begin{cases}
        \mathrm{true} & \text{if } t > u_{1,i}\\[2pt]
        (s \models p_{1,i}) \land \xi \in \operatorname{CPre}(W_{\Box,i}) & \text{if } l_{1,i} {\leq}\, t {\leq}\, u_{1,i}\\[2pt]
        \xi \in \operatorname{CPre}(W_{\Box,i}) & \text{if } t < l_{1,i}
    \end{cases}.
\end{equation}

\textbf{Reachability} $R_{\Diamond}$:
For $\varphi_{\Diamond,i} = \Diamond_{[l_{2,i},u_{2,i}]}\, p_{2,i}$ with timer~$x_{\mathrm{reach}}^i$:
\begin{equation}\label{eq:req_reach}
R'_{\Diamond}(\xi,\, p_{2,i}) =
\begin{cases}
\mathrm{true} & \!\!\!\!\text{if } p_{\mathrm{reach}}^i \\[8pt]
\begin{aligned}
    &\bigl(\mathcal{V}(v,p_{2,i}) \le u_{2,i} - x_{\mathrm{reach}}^i\bigr) \\[-2pt]
    &\land \xi \in \operatorname{CPre}(W_{\Diamond,i})
\end{aligned} & \!\!\!\!\begin{aligned}[t]
    &\text{if } \neg p_{\mathrm{reach}}^i \\[-2pt]
    &\land x_{\mathrm{reach}}^i < l_{2,i}
\end{aligned} \\[12pt]
\mathcal{V}(v,p_{2,i}) \le u_{2,i} - x_{\mathrm{reach}}^i & \!\!\!\!\begin{aligned}[t]
    &\text{if } \neg p_{\mathrm{reach}}^i \\[-2pt]
    &\land x_{\mathrm{reach}}^i \ge l_{2,i}
\end{aligned}
\end{cases}.
\end{equation}
The flag $p_{\mathrm{reach}}^i$ is initialized to~$\mathrm{false}$ and is set to $\mathrm{true}$ once $s \models p_{2,i}$ while $x_{\mathrm{reach}}^i \in [l_{2,i}, u_{2,i}]$.


\textbf{Response} $R_{\Rightarrow\Diamond}$: For $\varphi_{\Rightarrow\Diamond,i} = \Box(q_{3,i} \Rightarrow \Diamond_{[l_{3,i},u_{3,i}]}\, p_{3,i})$ with timer~$x_{\mathrm{resp}}^i$:
\begin{equation}\label{eq:req_resp}
R'_{\Rightarrow\Diamond}(\xi,\, p_{3,i}) =
\begin{cases}
\mathrm{true} & \!\!\!\!\begin{aligned}[t]
    &\text{if } s \not\models q_{3,i} \\[-2pt]
    &\lor\; p_{\mathrm{resp}}^i
  \end{aligned} \\[12pt]
\begin{aligned}
    &\textstyle\bigvee_{a \in \mathcal{A}} \bigwedge_{\xi' \in \delta(\xi,\, a)} \Bigl[ \xi' {\in} W_{\Rightarrow\Diamond,i} \\[-2pt]
    &\quad \land\; \textstyle\bigwedge_{j \neq i} \mathrm{Feas}(j,\,\xi') \Bigr]
\end{aligned}
& \!\!\!\!\begin{aligned}[t]
    &\text{if } s \models p_{3,i} \\[-2pt]
    &\land x_{\mathrm{resp}}^i \ge l_{3,i}
  \end{aligned} \\[16pt]
\begin{aligned}
    &\bigl(\mathcal{V}(v, p_{3,i}) \le u_{3,i} - x_{\mathrm{resp}}^i\bigr) \\[-2pt]
    &\;\land\; \xi \in \operatorname{CPre}(W_{\Rightarrow\Diamond,i})
\end{aligned}
& \!\!\!\!\text{otherwise}
\end{cases},
\end{equation}
where $\mathrm{Feas}(j, \xi') \triangleq (s' \not\models q_{3,j}) \lor p_{\mathrm{resp}}^{j\prime} \lor \bigl(\mathcal{V}(v', p_{3,j}) \le u_{3,j} - x_{\mathrm{resp}}^{j\prime}\bigr)$, with $s', v', p_{\mathrm{resp}}^{j\prime}$, and $x_{\mathrm{resp}}^{j\prime}$ denoting the components of $\xi'$, ensures that resolving task~$i$ preserves feasibility of concurrent tasks~$j$. The flag $p_{\mathrm{resp}}^i$ is set to $\mathrm{true}$ when $s \models p_{3,i}$ with $x_{\mathrm{resp}}^i \in [l_{3,i}, u_{3,i}]$, and resets to $\mathrm{false}$ upon $q_{3,i}$.


\textbf{Recurrence} $R_{\Box \Diamond}$: For $\varphi_{\Box\Diamond,i} = \Box\bigl(q_{4,i} \Rightarrow \Box_{[0,\,\bar{u}_{4,i}]} \Diamond_{[0,\,\hat{u}_{4,i}]}\, p_{4,i}\bigr)$, with timers $z_i$ and $y_i$ tracking the elapsed time for the outer window and the inner deadline, respectively. Depending on the activation flag $b_{\mathrm{rec}}^i$, the specification is classified into two internal modes: \text{IDLE} ($\neg b_{\mathrm{rec}}^i$) and \text{ACTIVE} ($b_{\mathrm{rec}}^i$):
\begin{equation}\label{eq:req_rec}
R'_{\Box\Diamond}(\xi,\, p_{4,i}) =
\begin{cases}
\begin{aligned}
    &(z_i > \bar{u}_{4,i}) \;\lor \\
    &\bigl[\mathcal{V}(v,p_{4,i}) \le \hat{u}_{4,i} - y_i \\
    &\;\land\; \xi \in \operatorname{CPre}(W_{\Box\Diamond,i})\bigr]
\end{aligned}
& \!\!\text{if } b_{\mathrm{rec}}^i \\
\xi \in \operatorname{CPre}(W_{\Box\Diamond,i})
& \!\!\text{if } \neg b_{\mathrm{rec}}^i
\end{cases}.
\end{equation}
The memory triple $(b_{\mathrm{rec}}^i,\, z_i,\, y_i)$ evolves upon each transition to a successor location $s'$ as follows:
\begin{equation}\label{eq:rec_update}
(b_{\mathrm{rec}}^{i\prime},\, z_i^{\prime},\, y_i^{\prime}) =
\begin{cases}
(\mathrm{true},\, 0,\, 0) & \!\text{if } \neg b_{\mathrm{rec}}^i \land s' \models q_{4,i} \\
(\mathrm{true},\, z_i,\, 0) & \!\text{if } b_{\mathrm{rec}}^i \land s' \models p_{4,i} \land z_i \le \bar{u}_{4,i} \\
(\mathrm{false},\, z_i,\, y_i) & \!\text{if } b_{\mathrm{rec}}^i \land z_i > \bar{u}_{4,i} \\
(b_{\mathrm{rec}}^i,\, z_i,\, y_i) & \!\text{otherwise}
\end{cases}.
\end{equation}

\textbf{Persistence} $R_{\Diamond\Box}$: For $\varphi_{\Diamond\Box,i} = \Box\bigl(q_{5,i} \Rightarrow \Diamond_{[0,\,\bar{u}_{5,i}]} \Box_{[0,\,\hat{u}_{5,i}]}\, p_{5,i}\bigr)$, with timers $z_i$ and $y_i$ tracking the elapsed time for the outer deadline and the holding duration, respectively. Depending on the activation mode $b_{\mathrm{act}}^i$ and holding mode $b_{\mathrm{hld}}^i$, the specification is classified into three internal modes: \text{IDLE} ($\neg b_{\mathrm{act}}^i$), \text{SEARCH} ($b_{\mathrm{act}}^i \land \neg b_{\mathrm{hld}}^i$), and \text{HOLD} ($b_{\mathrm{act}}^i \land b_{\mathrm{hld}}^i$):
\begin{equation}\label{eq:req_pers}
\begin{aligned}
    & R'_{\Diamond\Box}(\xi,\, p_{5,i}) =
    \\
    &\begin{cases}
    \begin{aligned}
        &(y_i \ge \hat{u}_{5,i}) \;\lor \\
        &\bigl[ (s {\models} p_{5,i}) \land \xi \in \operatorname{CPre}(W_{\Diamond\Box,i}) \bigr]
    \end{aligned}
    & \!\!\text{if } b_{\mathrm{act}}^i \land b_{\mathrm{hld}}^i \\
    \begin{aligned}
        &(z_i {\le} \bar{u}_{5,i}) \land \bigl(\mathcal{V}(v, p_{5,i}) {\le} \bar{u}_{5,i} {-} z_i\bigr) \\
        &\land \xi \in \operatorname{CPre}(W_{\Diamond\Box,i})
    \end{aligned}
    & \!\!\text{if } b_{\mathrm{act}}^i \land \neg b_{\mathrm{hld}}^i \\
    \xi \in \operatorname{CPre}(W_{\Diamond\Box,i})
    & \!\!\text{if } \neg b_{\mathrm{act}}^i
    \end{cases}.
\end{aligned}
\end{equation}
The memory tuple $(b_{\mathrm{act}}^i,\, b_{\mathrm{hld}}^i,\, z_i,\, y_i)$ is updated upon each transition to a successor location $s'$ according to the following rules:
\begin{equation}\label{eq:pers_update}
\begin{aligned}
&(b_{\mathrm{act}}^{i\prime},\, b_{\mathrm{hld}}^{i\prime},\, z_i^{\prime},\, y_i^{\prime}) = \\
&\begin{cases}
    (\mathrm{true}, \mathrm{false}, 0, 0) 
    & \!\text{if } \neg b_{\mathrm{act}}^i \land s' \models q_{5,i} \\
    (\mathrm{true}, \mathrm{true}, z_i, 0) 
    & \!\begin{aligned}[t]
        &\text{if } b_{\mathrm{act}}^i \land \neg b_{\mathrm{hld}}^i \land {} \\
        & \quad s' \models p_{5,i} \land z_i \le \bar{u}_{5,i} 
      \end{aligned} \\
    (\mathrm{false}, \mathrm{false}, z_i, y_i)
    & \!\begin{aligned}[t]
        &\text{if } (b_{\mathrm{act}}^i \land \neg b_{\mathrm{hld}}^i \land z_i > \bar{u}_{5,i}) \\
        &\quad \lor (b_{\mathrm{act}}^i \land b_{\mathrm{hld}}^i \land y_i \ge \hat{u}_{5,i})
      \end{aligned} \\
    (\mathrm{true}, \mathrm{false}, z_i, 0) 
    & \!\begin{aligned}[t]
        &\text{if } b_{\mathrm{act}}^i \land b_{\mathrm{hld}}^i \land {} \\
        &\quad s' \not\models p_{5,i} \land y_i < \hat{u}_{5,i} 
      \end{aligned} \\
    (b_{\mathrm{act}}^i, b_{\mathrm{hld}}^i, z_i, y_i) 
    & \!\text{otherwise}
\end{cases}.
\end{aligned}
\end{equation}
When triggered by $q_{5,i}$, the agent transitions from \text{IDLE} to \text{SEARCH}, navigating toward $p_{5,i}$ within the outer deadline $\bar{u}_{5,i}$. Once reached, \text{HOLD} maintains $p_{5,i}$ until $y_i \ge \hat{u}_{5,i}$, reverting to \text{SEARCH} if $p_{5,i}$ is violated. All three modes enforce recursive invariance via $\operatorname{CPre}$.

With all five requirement functions defined, Algorithm~\ref{alg:winning} computes $W_{\mathrm{total}}$ via hierarchical fixed-point iterations. Each stage refines the winning set of one specification type within the intersection of previously computed sets.


\begin{algorithm}[h]
\caption{\textsc{ComputeWinningSet}($\Xi$, $\varphi$, $G_Z$)}\label{alg:winning}

\SetKwInOut{Input}{Input}
\SetKwInOut{Output}{Output}

\Input{Augmented state space $\Xi$, agent specification $\varphi$, zone graph $G_Z$}
\Output{Total winning set $W_{\mathrm{total}}$}

Compute $\mathcal{V}(v,\, p_{\kappa,i})$ for all propositions\;

\textbf{Stage 1: Safety}\;
$W_{\Box} \leftarrow \Xi$;\quad $\hat{W} \leftarrow \emptyset$\;
\While{$W_{\Box} \neq \hat{W}$}{
    $\hat{W} \leftarrow W_{\Box}$\;
    \ForAll{$\xi \in \hat{W}$}{
        \lIf{$\neg R_{\Box}(\xi)$}{$W_{\Box} \leftarrow W_{\Box} \setminus \{\xi\}$}
    }
}

\textbf{Stage 2: Reachability}\;
$W_{\Diamond} \leftarrow W_{\Box}$;\quad $\hat{W} \leftarrow \emptyset$\;
\While{$W_{\Diamond} \neq \hat{W}$}{
    $\hat{W} \leftarrow W_{\Diamond}$\;
    \ForAll{$\xi \in \hat{W}$}{
        \lIf{$\neg R_{\Diamond}(\xi)$}{$W_{\Diamond} \leftarrow W_{\Diamond} \setminus \{\xi\}$}
    }
}

\textbf{Stage 3: Recurrence}\;
$W_{\Box\Diamond} \leftarrow W_{\Box} \cap W_{\Diamond}$;\quad $\hat{W} \leftarrow \emptyset$\;
\While{$W_{\Box\Diamond} \neq \hat{W}$}{
    $\hat{W} \leftarrow W_{\Box\Diamond}$\;
    \ForAll{$\xi \in \hat{W}$}{
        \lIf{$\neg R_{\Box\Diamond}(\xi)$}{$W_{\Box\Diamond} \leftarrow W_{\Box\Diamond} \setminus \{\xi\}$}
    }
}

\textbf{Stage 4: Persistence}\;
$W_{\Diamond\Box} \leftarrow W_{\Box\Diamond}$;\quad $\hat{W} \leftarrow \emptyset$\;
\While{$W_{\Diamond\Box} \neq \hat{W}$}{
    $\hat{W} \leftarrow W_{\Diamond\Box}$\;
    \ForAll{$\xi \in \hat{W}$}{
        \lIf{$\neg R_{\Diamond\Box}(\xi)$}{$W_{\Diamond\Box} \leftarrow W_{\Diamond\Box} \setminus \{\xi\}$}
    }
}

\textbf{Stage 5: Response}\;
$W_{\Rightarrow\Diamond} \leftarrow \emptyset$;\quad $\hat{W} \leftarrow W_{\Diamond\Box}$\;
\While{$W_{\Rightarrow\Diamond} \neq \hat{W}$}{
    $\hat{W} \leftarrow W_{\Rightarrow\Diamond}$\;
    \ForAll{$\xi \in W_{\Diamond\Box} \setminus \hat{W}$}{
        \lIf{$R_{\Rightarrow\Diamond}(\xi)$}{$W_{\Rightarrow\Diamond} \leftarrow W_{\Rightarrow\Diamond} \cup \{\xi\}$}
    }
}

\Return $W_{\mathrm{total}} \leftarrow W_{\Box} \cap W_{\Diamond} \cap W_{\Box\Diamond} \cap W_{\Diamond\Box} \cap W_{\Rightarrow\Diamond}$\;
\end{algorithm}

\subsection{Constraint Function}
The winning set characterizes states from which all specifications can be satisfied. At runtime, the constraint function prevents the behavior tree from executing transitions that would leave this set.

\begin{definition}[Constraint Function]
    The \emph{constraint function} $\rho\colon \Xi \times \mathbb{R}_{+} \to \{\mathrm{true}, \mathrm{false}\}$ is:
    \begin{equation}\label{eq:constraint}
        \rho(\xi, \Delta) = \rho_{\mathrm{dead}}(\xi, \Delta) \;\land\; \rho_{\mathrm{env}}(\xi),
    \end{equation}
    with:
    \begin{align}
        \rho_{\mathrm{dead}}(\xi, \Delta) &= \textstyle\bigwedge_{t \in [0,\Delta]} \bigl[\xi{+}t \in W_{\mathrm{total}} \;\land\; \nu{+}t \models I(s)\bigr], \label{eq:rho_dead}\\
        \rho_{\mathrm{env}}(\xi) &= \textstyle\bigvee_{a \in \mathcal{A}}\; \bigwedge_{\xi' \in \delta(\xi,a)}\; \xi' \in W_{\mathrm{total}}. \label{eq:rho_env}
    \end{align}
\end{definition}

The deadline guard~$\rho_{\mathrm{dead}}$ verifies that the system remains in the winning set and satisfies the location invariant~$I(s)$ (with current clock valuation~$\nu$) throughout a proposed waiting period~$\Delta$. The environment guard~$\rho_{\mathrm{env}}$ ensures that at least one action guarantees all possible successors remain winning, consistent with Assumption~\ref{assump:No Deadlock}.

\subsection{Behavior Tree Construction}

Given the requirement function $R$ and constraint function $\rho$, a BT $\mathcal{B}$ ensuring $\pi_{\mathcal{B}} \models \varphi$ is synthesized. As shown in Fig.~\ref{fig:bt_arch}, the root node groups subtrees by type. Safety is enforced via $R$ and $\rho$ blocking transitions outside $W_{\Box}$.

\begin{figure}[h]
    \centering
    \includegraphics[width=0.85\linewidth]{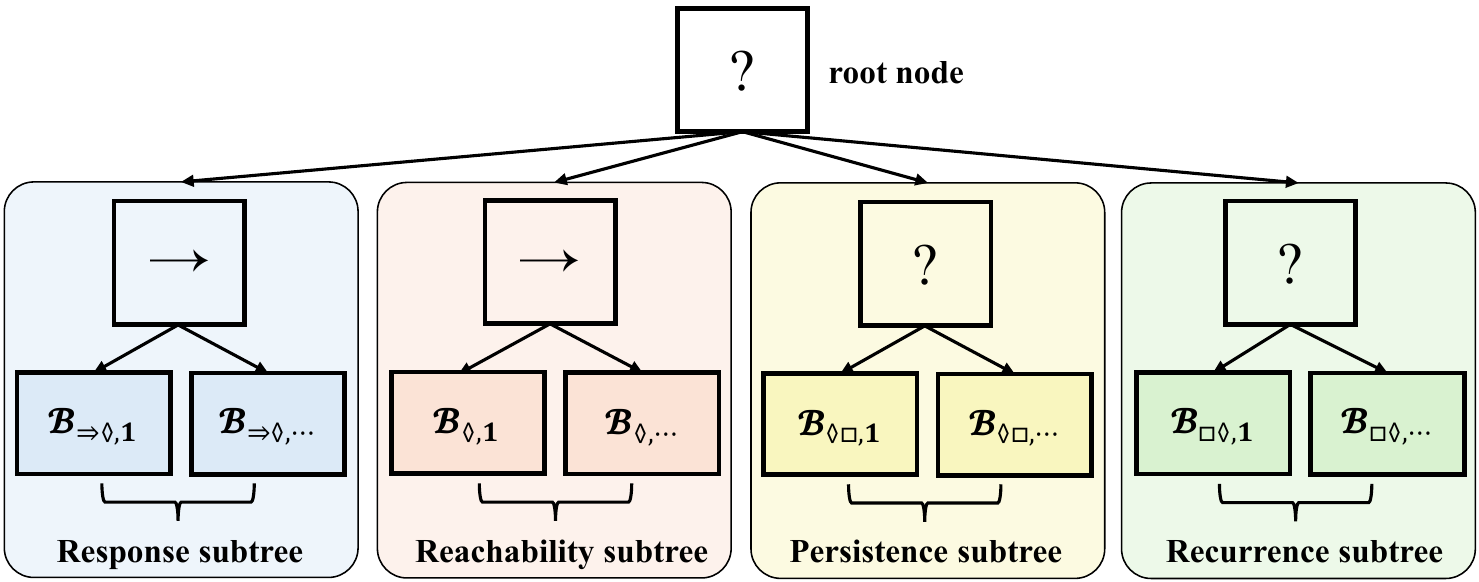}
    \caption{Overall structure of the synthesized Behavior Tree.}
    \label{fig:bt_arch}
\end{figure}

All four subtree types share two common branches. The \emph{action-selection} branch (as shown in Fig.~\ref{fig:bt_reach}) evaluates each candidate successor $\xi' \in \delta(\xi, a)$, retains those satisfying $\xi' \in W_{\mathrm{total}} \land \rho_{\mathrm{env}}(\xi')$, and executes
\begin{equation}\label{eq:action_select}
    a^* = \argmin_{\substack{a \in \mathcal{A}}} \mathcal{V}\!\bigl(v',\; p_{\kappa,i}\bigr),
\end{equation}
where $v'$ denotes the zone-graph node of the successor state $\xi'$, $p_{\kappa,i}\in \mathit{AP}$. The \emph{safe-wait} branch (as shown in Fig.~\ref{fig:bt_reach}) verifies $\rho_{\mathrm{dead}}(\xi, \Delta)$ and delays for~$\Delta$ time units when no discrete action is preferred. Beyond these shared components, each specification type requires a distinct subtree structure.

\textbf{Reachability subtree} $\mathcal{B}_{\Diamond,i}$: For $\varphi_{\Diamond,i} = \Diamond_{[l_{2,i},u_{2,i}]}\, p_{2,i}$, the subtree is illustrated in Fig.~\ref{fig:bt_reach}.
Once $s \models p_{2,i}$ holds within $[l_{2,i}, u_{2,i}]$, branch~2 sets $p_{\mathrm{reach}}^i$ to $\mathrm{true}$; all subsequent ticks return \emph{success} via branch~1. Otherwise, the \emph{action-selection} or \emph{safe-wait} branches execute. When $x_{\mathrm{reach}}^i > u_{2,i}$, no branch succeeds and the subtree returns \emph{failure}.

\begin{figure}[h]
    \centering
    \includegraphics[width=0.9\linewidth]{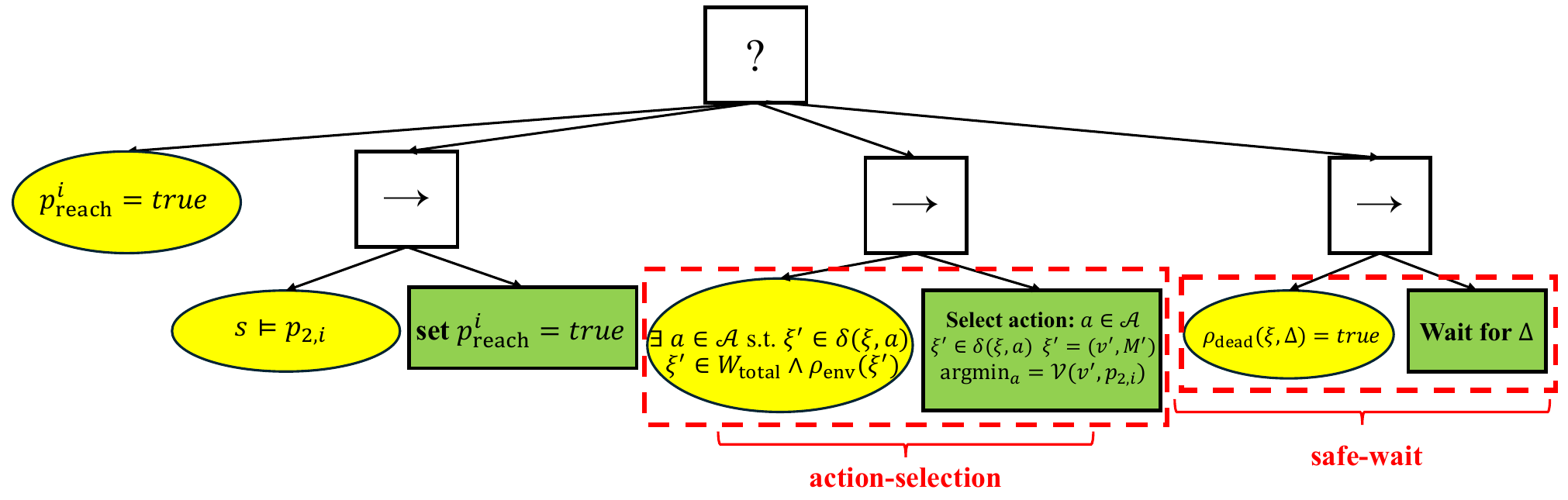}
    \caption{Structure of a Reachability subtree $\mathcal{B}_{\Diamond,i}$}
    \label{fig:bt_reach}
\end{figure}

\textbf{Response subtree} $\mathcal{B}_{\Rightarrow \Diamond,i}$: For $\varphi_{\Rightarrow\Diamond,i} = \Box(q_{3,i} \Rightarrow \Diamond_{[l_{3,i},u_{3,i}]}\, p_{3,i})$, the subtree is illustrated in Fig.~\ref{fig:bt_resp}.
Branch~1 returns \emph{success} when $q_{3,i}$ is inactive. When active, branch~2 sets $p_{\mathrm{resp}}^i$ to $\mathrm{true}$ if $s \models p_{3,i}$ or $p_{\mathrm{resp}}^i$ is already set.

\begin{figure}[h]
    \centering
    \includegraphics[width=0.9\linewidth]{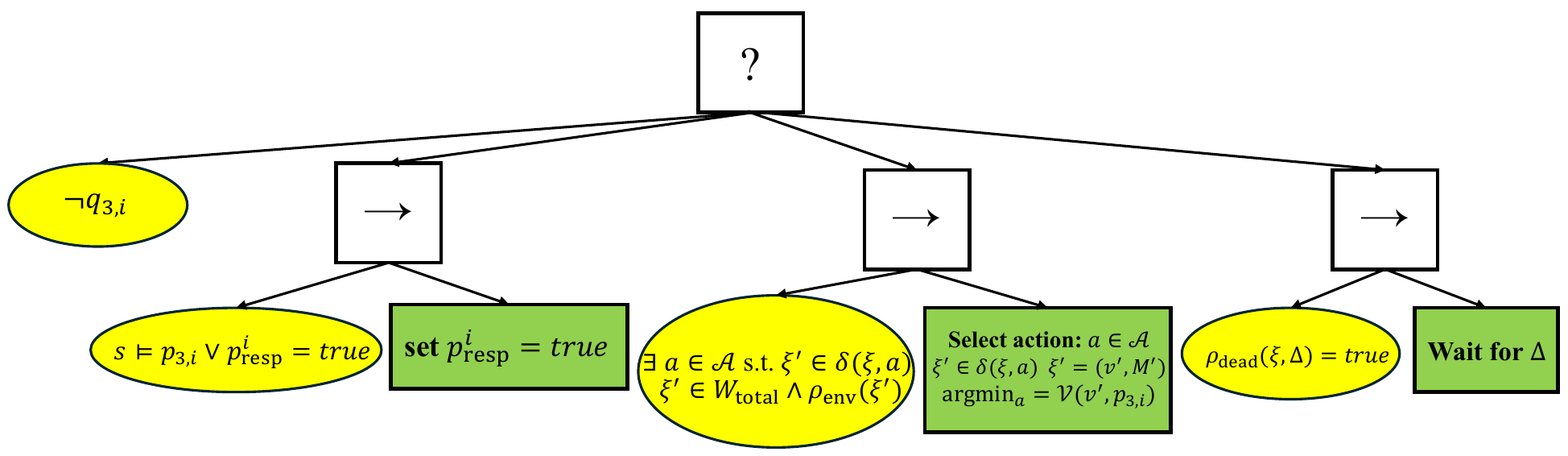}
    \caption{Structure of a Response subtree $\mathcal{B}_{\Rightarrow \Diamond,i}$}
    \label{fig:bt_resp}
\end{figure}

\textbf{Recurrence subtree} $\mathcal{B}_{\Box \Diamond,i}$: For $\varphi_{\Box\Diamond,i} = \Box(q_{4,i} \Rightarrow \Box_{[0,\bar{u}_{4,i}]}\Diamond_{[0,\hat{u}_{4,i}]}\, p_{4,i})$, the subtree coordinates IDLE ($\neg b_{\mathrm{rec}}^i$) and ACTIVE ($b_{\mathrm{rec}}^i$) modes via a seven-branch fallback (Fig.~\ref{fig:bt_rec}). Branches~1 and~3 yield via $\mathrm{S{\to}R}$ on timeout ($z_i > \bar{u}_{4,i}$). Branches~2 and~6 navigate toward~$p_{4,i}$ in IDLE and ACTIVE modes, respectively. Branch~4 returns \emph{success} upon reaching the target. Branch~5 evaluates deadline feasibility via an $\mathrm{Inv}$ decorator applied to ($\mathcal{V}(v,p_{4,i}) \le \hat{u}_{4,i} - y_i$), yielding \emph{failure} via $\mathrm{S{\to}F}$ if violated. Branch~7 provides a default wait. Mode transitions follow~\eqref{eq:rec_update}.

\begin{figure}[h]
    \centering
    \includegraphics[width=0.98\linewidth]{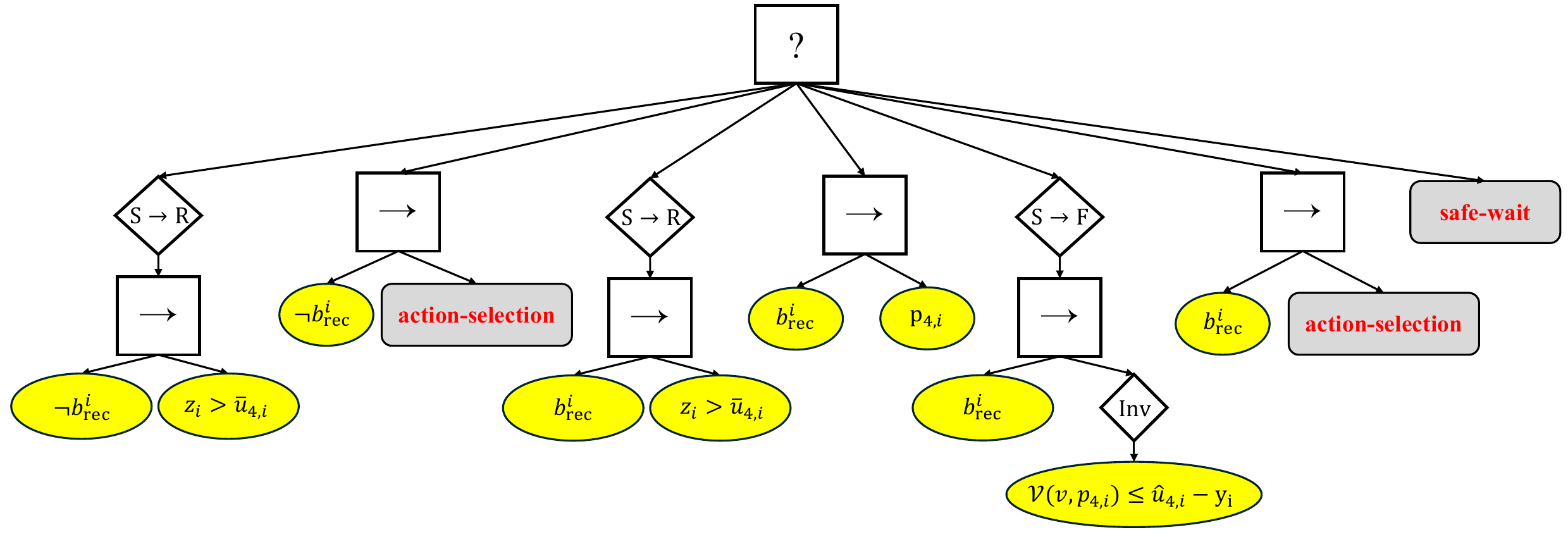}
    \caption{Structure of a Recurrence subtree $\mathcal{B}_{\Box \Diamond,i}$}
    \label{fig:bt_rec}
\end{figure}

\textbf{Persistence subtree} $\mathcal{B}_{\Diamond\Box,i}$: For $\varphi_{\Diamond\Box,i} = \Box(q_{5,i} \Rightarrow \Diamond_{[0,\bar{u}_{5,i}]} \Box_{[0,\hat{u}_{5,i}]}\, p_{5,i})$, the subtree manages three internal modes via an eight-branch fallback (Fig.~\ref{fig:bt_persist}). In IDLE mode ($\neg b_{\mathrm{act}}^i$), branch~1 yields control via an $\mathrm{S{\to}F}$ decorator when $q_{5,i}$ is inactive, while branch~2 initiates SEARCH upon activation. In HOLD mode ($b_{\mathrm{act}}^i \land b_{\mathrm{hld}}^i$), branch~3 uses $\mathrm{S{\to}F}$ to signal completion ($y_i \ge \hat{u}_{5,i}$), branch~5 re-navigates if displaced from~$p_{5,i}$, and branch~6 waits in place. In SEARCH mode ($b_{\mathrm{act}}^i \land \neg b_{\mathrm{hld}}^i$), branch~4 yields on timeout ($z_i \ge \bar{u}_{5,i}$) via $\mathrm{S{\to}F}$, branch~7 waits upon reaching $p_{5,i}$, and branch~8 navigates toward the target. All mode transitions follow~\eqref{eq:pers_update}.

\begin{figure}[h]
    \centering
    \includegraphics[width=\linewidth]{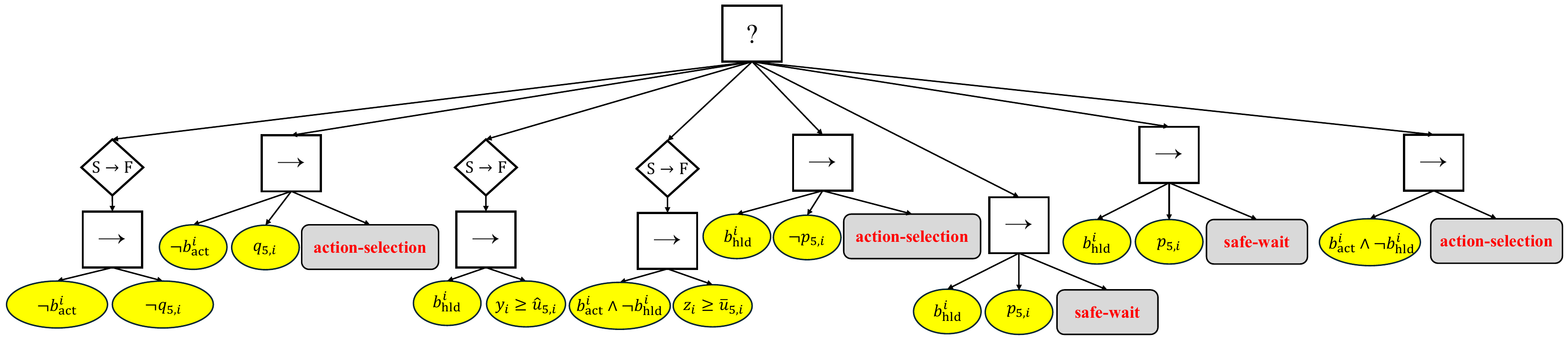}
    \caption{Structure of a Persistence subtree $\mathcal{B}_{\Diamond\Box,i}$}
    \label{fig:bt_persist}
\end{figure}

\section{THEORETICAL ANALYSIS}

This section establishes the formal guarantees and computational complexity of the proposed synthesis framework under Assumptions~\ref{assump:Non Zeno}--\ref{assump: Bounded Adversial}. Recursive feasibility is first proven through the forward invariance of the winning set, followed by the verification of monotonic progress that guarantees liveness under timing constraints.

\begin{lemma}[Forward Invariance]\label{lem:safety}
The execution of the synthesized Behavior Tree with the constraint function $\rho$ ensures that the system state $\xi(t)$ remains within the total winning set $W_{\mathrm{total}}$ for $t \geq 0$, satisfying safety specifications.
\end{lemma}

\begin{proof}
By Assumption~\ref{assump:Initial Consistency}, $\xi_0 \in W_{\mathrm{total}}$. We proceed by induction over the BT execution ticks indexed by $k \ge 0$. Suppose at tick $k$ (time $t_k$), $\xi_k \in W_{\mathrm{total}}$. The BT produces either a \emph{time transition} (wait) or a \emph{discrete transition} (action):

1) \textit{Time Transition:} The \emph{safe-wait} branch permits waiting for $\Delta$ iff $\rho_{\mathrm{dead}}(\xi_k, \Delta)$ holds. By Eq.~\eqref{eq:rho_dead}, the continuous evolution $\xi(t_k + \tau)$ remains strictly in $W_{\mathrm{total}}$ for all $\tau \in [0, \Delta]$, ensuring the subsequent state $\xi_{k+1} \in W_{\mathrm{total}}$.

2) \textit{Discrete Transition:} The \emph{action-selection} branch~\eqref{eq:action_select} computes an action $a^*$ filtered by $\rho_{\mathrm{env}}(\xi_k)$. By Eq.~\eqref{eq:rho_env}, all possible successor states under environmental non-determinism satisfy $\xi' \in W_{\mathrm{total}}$. As discrete transitions are instantaneous, $\xi_{k+1} \in W_{\mathrm{total}}$.

Both branches prevent violating $W_{\mathrm{total}}$, so induction guarantees $\xi(t) \in W_{\mathrm{total}}$ for all time $t \ge 0$. Consequently, each safety requirement $R_{\Box,i}$ remains $\mathrm{true}$, enforcing $s(t) \models p_{1,i}$ over $[l_{1,i}, u_{1,i}]$.
\end{proof}

\begin{lemma}[Monotonic Progress and Timing Guarantees]\label{lem:progress}
Within the set $W_{\mathrm{total}}$, the induced policy $\pi_{\mathcal{B}}$ guarantees that all active liveness tasks (Reachability, Response, Recurrence, Persistence) fulfill their specifications within the timing constraints without cross-specification interference.
\end{lemma}

\begin{proof}
Consider an active liveness specification $i$ bounded by $[l_i,u_i]$. By Lemma~\ref{lem:safety}, $\xi(t) \in W_{\mathrm{total}}$, continuously enforcing $0 \le \mathcal{V}(v(t), p_i) \le u_i - x_i(t)$. As $x_i(t) \to u_i$, the upper bound decreases to zero. Therefore, $\lim_{x_i(t) \to u_i} \mathcal{V} = 0$. 

Since $\mathcal{V}$ computes the minimum time-to-go on a finite graph with positive edge costs, its reduction necessitates discrete spatial transitions. Under Assumptions~\ref{assump:No Deadlock}--\ref{assump: Bounded Adversial}, the environment cannot perpetually block descent actions. Thus, the BT forces the system to reach the target at some time $t^* \leq u_i$, achieving $s(t^*) \models p_i$. If reached prematurely ($t^* < l_i$), the $\operatorname{CPre}$ operator ensures a safe wait until $l_i$.

For concurrent tasks, $W_{\mathrm{total}} = \bigcap_{\kappa,i} W_{\kappa, i}$~\eqref{eq:winning_total} prevents interference. Action selection~\eqref{eq:action_select} restricts any transition $\xi \to \xi'$ to $W_{\mathrm{total}}$. Therefore, fulfilling specification $i$ guarantees the successor $\xi'$ remains valid for all other tasks $j$, intrinsically preserving their physical feasibility ($\mathcal{V}(v', p_j) \le u_j - x_j$) and safety.
\end{proof}

\begin{theorem}[Correct-by-Construction Guarantee]\label{thm:correctness}
Given a timed transition system $\mathcal{T}$ and an STL specification $\varphi$, the synthesized Behavior Tree $\mathcal{B}$ generates an induced policy $\pi_{\mathcal{B}}$ that satisfies Problem~\ref{prob:Problem_synthesis}, i.e., $\pi_{\mathcal{B}} \models \varphi$.
\end{theorem}

\begin{proof}
The proof follows directly from Lemma~\ref{lem:safety} and Lemma~\ref{lem:progress}. The hierarchical BT structure strictly confines system execution within the total winning set $W_{\mathrm{total}}$, structurally preventing safety violations. Concurrently, the value-function-guided action selection enforces monotonically decreasing time-to-go for all active specification components, fulfilling the quantitative liveness bounds. Therefore, any execution trace under $\pi_{\mathcal{B}}$ inherently conforms to $\varphi$.
\end{proof}

\begin{figure*}[t]
    \centering
    \includegraphics[width=0.96\linewidth]{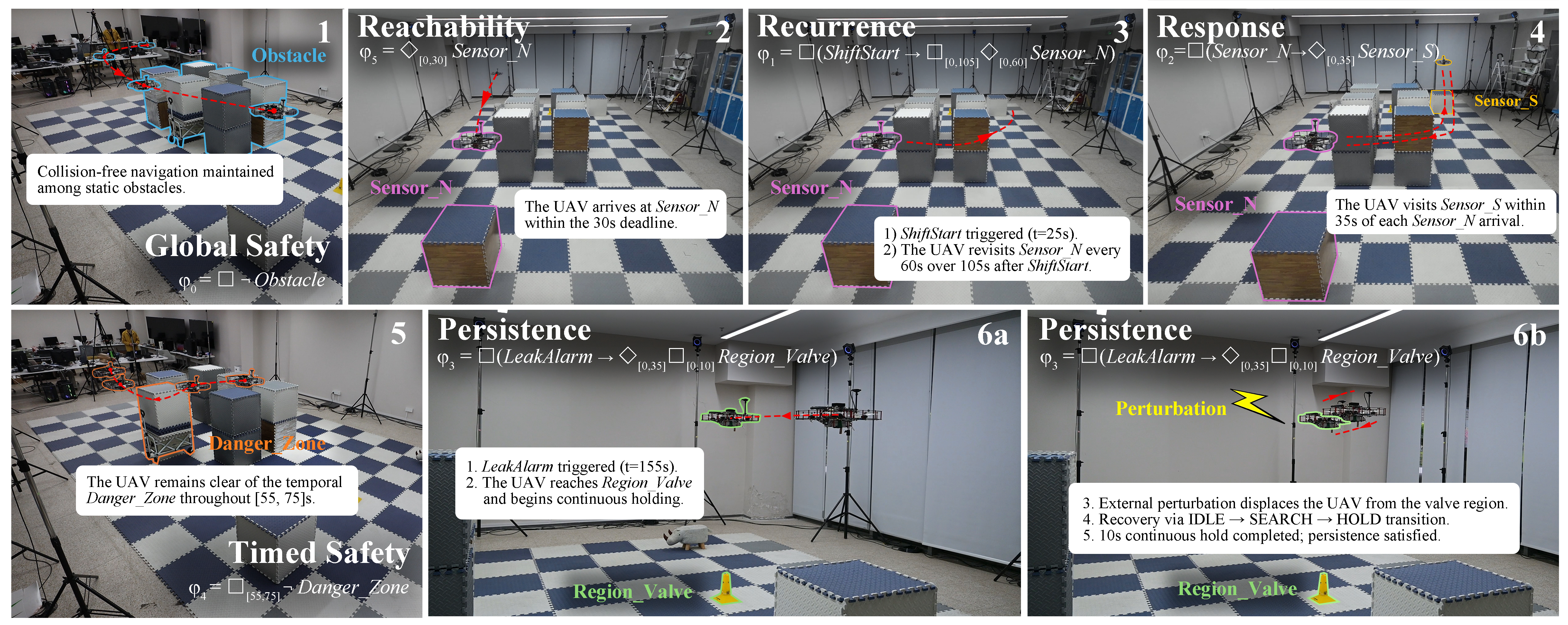}
    \caption{Key frames of the physical experiment showing how the UAV satisfies each specification.}
    \label{fig:exp_key_frame}
\end{figure*}

\begin{theorem}[Computational Complexity]\label{thm:complexity}
Let $|V|$ and $|E|$ be the number of nodes and edges of the zone graph $G_Z$, $|\mathcal{A}|$ the number of actions, $|\varphi|$ the number of specifications, and $U = \max_j u_j$ the maximum saturated deadline. The offline synthesis time complexity is
\[
O\bigl(|\varphi| \cdot |E| \cdot \log |V| + |\varphi| \cdot |\Xi|^2 \cdot |\mathcal{A}|\bigr),
\]
where $|\Xi| = O(|V| \cdot 2^{|\varphi|} \cdot U^{|\varphi|})$ bounds the augmented state space size.
\end{theorem}

\begin{proof}
The value function $\mathcal{V}$ is calculated via Dijkstra for all $|\varphi|$ specifications in $O(|\varphi| \cdot |E| \cdot \log |V|)$ time. The augmented space $\Xi$ couples the zone graph $|V|$ with boolean memory ($2^{|\varphi|}$) and saturated timers ($U^{|\varphi|}$). The winning set is constructed using hierarchical fixed-point iterations. Convergence occurs in at most $|\Xi|$ iterations; each iteration evaluates $|\Xi|$ states. For each state, checking the requirement function involves computing $\operatorname{CPre}$ (Eq.~\ref{eq:cpre-}), which evaluates the validity of all successors generated by available actions, bounded by $O(|\mathcal{A}|)$. Thus, each stage takes $O(|\Xi|^2 \cdot |\mathcal{A}|)$, yielding an overall bound of $O(|\varphi| \cdot |\Xi|^2 \cdot |\mathcal{A}|)$.
\end{proof}


\section{EXPERIMENTS}

This section evaluates the proposed framework across two domains: grid-world simulations to quantify scalability and verify correctness over 30 test scenarios, and a physical quadrotor experiment to validate hardware execution.

\subsection{Experimental Setup and Metrics}

Simulations run on an Intel i9-13980HX CPU. The physical experiment uses a PX4 quadrotor with a Jetson Orin Nano running ROS~2 Jazzy (Fig.~\ref{fig:exp_setup}). Trajectories are evaluated via quantitative STL robustness $\rho$~\cite{donze2010robust}. For each atomic proposition, the signed Euclidean distance $\mu$ is computed from the robot to the proposition region boundary. The specification robustness $\rho$ is then obtained by recursively propagating $\mu$ through STL temporal operators ($\Box$, $\Diamond$, etc.), and a specification $\varphi$ is satisfied iff its initial robustness degree $\rho_0 \triangleq \rho(\sigma,\varphi,0) {\ge} 0$, where $\sigma$ is the observed trajectory signal. Four metrics are assessed: 1)~\emph{Satisfaction rate} $S$: the fraction of trials satisfying all specifications; 2)~\emph{Mean robustness} $\bar{\rho}$: the average $\rho_0$ across all specifications and trials, reflecting the execution safety margin; 3)~\emph{Synthesis time} $T_{\mathrm{off}}$: the offline time to construct the graph, value functions, and winning set; and 4)~\emph{State-space size} ($|G_Z|$, $|W_\mathrm{total}|$): the number of graph nodes and valid winning states, indicating memory footprint.


\subsection{Simulation Results}

A total of 30 scenarios are evaluated across five categories (six scenarios each): 1)~grid scale ($10{\times}10$--$35{\times}35$), 2)~specification count ($2$--$8$ specifications), 3)~spec combinations (ablating safety, reachability, response, recurrence, persistence), 4)~timing constraints ($2.0{\times}$--$1.1{\times}$ shortest-path time), and 5)~obstacle density ($0\%$--$25\%$). Table~\ref{tab:sim_results} reports category means across 10 randomized trials per scenario. It shows that the framework achieves near-perfect satisfaction ($S {\ge} 0.96$) and strictly positive robustness ($\bar{\rho} {>} 0$) across all categories. Specification count and spec-combination categories reach $S {=} 1.0$, confirming that the synthesis correctly handles all five specification types simultaneously. The only notable degradation occurs when timing constraints approach the shortest-path limit: the tightest deadline reduces per-scenario satisfaction to $S {=} 0.86$. The offline synthesis time $T_{\mathrm{off}}$ scales primarily with workspace size (from $1.78$\,s at $10{\times}10$ to $25.02$\,s at $35{\times}35$), while varying timing constraints or obstacle density incurs negligible additional cost.

\begin{table}[h]
\centering
\caption{Mean simulation results across five test categories.}
\label{tab:sim_results}
\setlength{\tabcolsep}{3.5pt}
\begin{tabular}{@{}llccccc@{}}
\toprule
\textbf{Category} & \textbf{Range} & $|G_Z|$ & $|W_{\mathrm{total}}|$ & $T_{\mathrm{off}}$(s) & $\bar{\rho}$ & $S$ \\
\midrule
Grid scale & $10{\times}10$--$35{\times}35$ & 492.0 & 979.0 & 11.70 & 0.195 & 0.990 \\
Spec count & $2$--$8$ spec & 339.5 & 621.7 &  7.93 & 0.206 & 1.000 \\
Spec combo & $2$--$5$ spec types & 339.5 & 565.2 &  3.92 & 0.236 & 1.000 \\
Deadline   & $2.0{\times}$--$1.1{\times}$ & 339.4 & 665.0 &  7.86 & 0.085 & 0.960 \\
Obstacles  & $0\%$--$25\%$ & 349.3 & 695.5 &  7.88 & 0.497 & 0.990 \\
\bottomrule
\end{tabular}
\end{table}

\subsection{Physical Experiment}

Fig.~\ref{fig:exp_setup} illustrates a $4\,\mathrm{m}{\times}8\,\mathrm{m}$ indoor UAV gas monitoring and sealing scenario, discretized into a $27{\times}14$ grid ($0.3\,\mathrm{m}$ resolution) at a $1.5\,\mathrm{m}$ altitude. The mission concurrently enforces six STL specifications: global safety $\varphi_0{=}\Box\neg\mathit{Obstacle}$, timed safety $\varphi_4{=}\Box_{[55, 75]}\neg\mathit{Danger\_Zone}$, initial reachability $\varphi_5{=}\Diamond_{[0, 30]}\mathit{Sensor\_N}$, periodic patrol $\varphi_1{=}\Box(\mathit{ShiftStart} \Rightarrow \Box_{[0, 105]}\Diamond_{[0, 60]}\mathit{Sensor\_N})$, secondary response confirmation $\varphi_2{=}\Box(\mathit{Sensor\_N} \Rightarrow \Diamond_{[0, 35]}\mathit{Sensor\_S})$, and emergency persistence sealing $\varphi_3{=}\Box(\mathit{LeakAlarm} \Rightarrow \Diamond_{[0, 35]}\Box_{[0, 10]}\mathit{Region\_Valve})$. Fig.~\ref{fig:exp_key_frame} shows key frames from the physical experiment.

\begin{figure}[h]
    \centering
    \includegraphics[width=\linewidth]{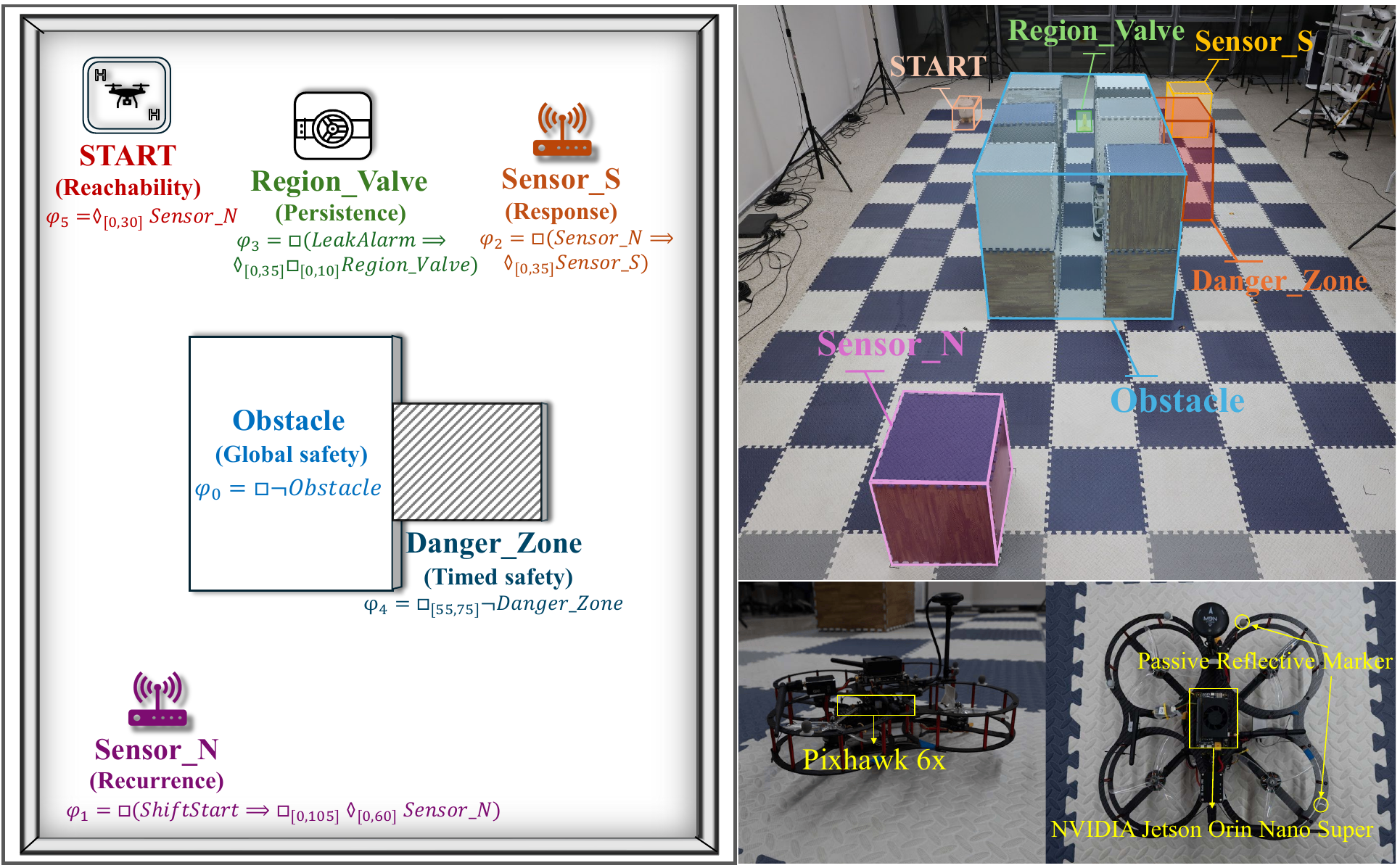}
    \caption{Gas monitoring and sealing scenario with STL specifications.}
    \label{fig:exp_setup}
\end{figure}

\begin{figure}[h]
    \centering
    \includegraphics[width=\linewidth]{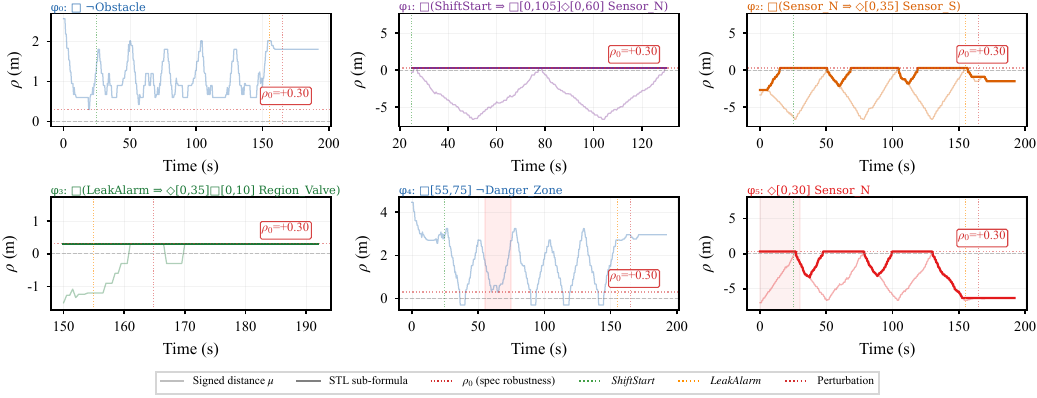}
    \caption{Per-specification STL robustness $\rho(t)$. Thin: signed distance $\mu$; bold: sub-formula robustness; dashed: $\rho_0$; dotted verticals: timed events and perturbation.}
    \label{fig:main_experiment}
\end{figure}

Offline synthesis for the six-specification scenario completes on the onboard computer prior to flight. A complete $192.0$\,s physical trial successfully satisfies all six specifications. Fig.~\ref{fig:main_experiment} illustrates the quantitative STL robustness (spatial margin in meters) evaluated over the UAV's trajectory. All specifications maintain strictly positive robustness. The timed-safety specification ($\varphi_4$) confirms that the UAV strictly avoids the temporary danger zone during $[55, 75]$\,s. Notably, the system successfully handles a physical perturbation introduced at $t {=} 164.9$\,s while executing the persistence specification $\varphi_3$. The synthesized Behavior Tree dynamically drives the UAV back to the target $Region\_Valve$, ensuring the contiguous $10$\,s holding requirement is ultimately met within the $35$\,s deadline. In conclusion, the proposed framework successfully manages multiple concurrent STL specifications and external disturbances, achieving rigorous correct-by-construction guarantees alongside physical execution.

\section{CONCLUSIONS}

This letter presented a framework that synthesized correct-by-construction Behavior Trees from Signal Temporal Logic specifications over timed transition systems. Through zone-graph abstraction, augmented states, and a value-function-guided hierarchical fixed-point algorithm, the framework provided formal correctness guarantees for STL specifications. Simulations across 30 scenarios and a physical quadrotor experiment with six concurrent specifications validated scalability, correctness, and practical deployability under disturbances. Future work targets online replanning for unknown environments and multi-robot coordination.


\bibliographystyle{IEEEtranBST/IEEEtran}

\bibliography{main}

\end{document}